\documentclass[twoside]{article}

\usepackage{tikz}
\usepackage{todonotes}
\usepackage[normalem]{ulem}
\usepackage{enumitem}
\usepackage{algpseudocode}
\usepackage{algorithm}
\usepackage[arxiv]{aistats2025}
\usepackage{amsfonts}
\usepackage{amsmath}
\usepackage{comment}
\usepackage{hyperref}
\usepackage{cleveref}
\usepackage{amsthm}
\usepackage[round]{natbib}

\newcommand{\one}[1]{\mathbf{1}_{#1}}
\newcommand{\dist}{\mathrm{dist}_G}
\newcommand{\comm}[1]{\textcolor{magenta}{#1}}

\newcommand{\Pp}{\mathbb{P}}
\newcommand{\chif}{\chi_{\scriptscriptstyle f}^{\scriptscriptstyle (d)}}

\newcommand{\B}{\mathcal{B}}
\newcommand{\R}{\mathbb{R}}

\newcommand{\klinf}{\mathrm{K\!L}_{\inf{}}}
\newcommand{\pleq}{\leq_{\delta}}
\newcommand{\Ell}{\mathcal{L}}
\newcommand{\Pn}{\hat{\mathrm{P}}_n}
\newcommand{\Reg}{\mathrm{R}}
\newcommand{\KL}{\mathrm{K\!L}}
\newcommand{\kl}{\mathrm{kl}}

\begin{document}

\newtheorem{theorem}{Theorem}
\newtheorem{corollary}{Corollary}
\newtheorem{proposition}{Proposition}
\newtheorem{lemma}{Lemma}
\newtheorem{definition}{Definition}
\newtheorem{remark}{Remark}
\newtheorem{assumption}{Assumption}

\newenvironment{sproof}{%
  \renewcommand{\proofname}{Sketch Proof}\proof}{\endproof}

\newcommand{\blue}[1]{\textcolor{blue}{#1}}
\newcommand{\mm}[1]{\textcolor{cyan}{#1}}
\renewcommand{\hat}{\widehat}

\newcommand{\bgen}{\overline{\gen}(W_n,S_n)}
\newcommand{\sm}{\gamma}
\newcommand{\dd}{\mathrm{d}}
\newcommand{\intW}{\int_{\Ww}}

\newcommand{\E}{\mathbb{E}}
\newcommand{\K}{\mathbb{K}}
\newcommand{\NN}{\mathbb{N}}
\newcommand{\Q}{\mathbb{Q}}
\newcommand{\real}{\mathbb{R}}
\newcommand{\T}{\mathbb{T}}

\newcommand{\Gen}{\mathrm{Gen}}
\newcommand{\Regret}{\mathrm{Regret}}

\newcommand{\A}{\mathcal{A}}
\newcommand{\D}{\mathcal{D}}
\newcommand{\F}{\mathcal{F}}
\newcommand{\N}{\mathcal{N}}
\newcommand{\W}{\mathcal{W}}
\newcommand{\X}{\mathcal{X}}
\newcommand{\Y}{\mathcal{Y}}
\newcommand{\Z}{\mathcal{Z}}

\newcommand{\C}{\mathcal{C}}
\newcommand{\Dw}{\mathcal{D}}
\newcommand{\Sw}{\mathcal{S}}
\newcommand{\Pw}{\mathcal{P}}
\newcommand{\Rw}{\mathcal{R}}
\newcommand{\Ww}{\mathcal{W}}
\newcommand{\HH}{\mathcal{H}}
\newcommand{\DD}[2]{\mathcal{D}\pa{#1\middle\|#2}}
\newcommand{\DDh}[2]{\mathcal{B}_{\bH}\pa{#1\middle\|#2}}
\newcommand{\DDp}[2]{\mathcal{D}_{\ell_p}\pa{#1\middle\|#2}}
\newcommand{\DDLC}[2]{\mathcal{D}_{\mathrm{LC}}\pa{#1\middle\|#2}}
\newcommand{\tDD}[2]{\wt{\mathcal{D}}\pa{#1\middle\|#2}}
\newcommand{\DDR}[2]{\mathcal{D}_R\pa{#1\middle\|#2}}
\newcommand{\DDsigma}[2]{\mathcal{D}_{\sm}\pa{#1\middle\|#2}}
 
\newcommand{\DDKLb}[2]{\mathcal{D}_{\mathrm{KL}}\bigl(#1\bigm\|#2\bigr)}
\newcommand{\kldiv}[2]{d_{\mathrm{kl}}\pa{#1\middle\|#2}}
\newcommand{\DDPhi}[2]{\mathcal{B}_\Phi\pa{#1\middle\|#2}}
\newcommand{\DDalpha}[2]{\mathcal{D}_{\alpha}\pa{#1\middle\|#2}}
\newcommand{\DDpsi}[2]{\mathcal{D}_{\psi}\pa{#1\middle\|#2}}
\newcommand{\DDchi}[2]{\mathcal{D}_{\chi^2}\pa{#1\middle\|#2}}
\newcommand{\DDphi}[2]{\mathcal{D}_{\varphi}\pa{#1\middle\|#2}}
\newcommand{\DDPhiS}[2]{\mathcal{D}_{\Phi_S}\pa{#1\middle\|#2}}
\newcommand{\DDtPhi}[2]{\mathcal{D}_{\tPhi}\pa{#1\middle\|#2}}
\newcommand{\DDS}[2]{\mathcal{D}^S\pa{#1\middle\|#2}}
\newcommand{\DDc}[3]{\mathcal{D}^{#3}\pa{#1\middle\|#2}}
\newcommand{\tDDS}[2]{\wt{\mathcal{D}}^S\pa{#1\middle\|#2}}
\newcommand{\DDf}[2]{\mathcal{D}_\varphi\pa{#1\middle\|#2}}
\newcommand{\DDU}[2]{\mathcal{D}_U\pa{#1\middle\|#2}}
\newcommand{\OO}{\mathcal{O}}
\newcommand{\tOO}{\wt{\OO}}
\newcommand{\trace}[1]{\mbox{trace}\left(#1\right)}
\newcommand{\II}[1]{\mathbb{I}_{\left\{#1\right\}}}
\newcommand{\PP}[1]{\mathbb{P}\left[#1\right]}
\newcommand{\PPZ}[1]{\mathbb{P}_Z\left[#1\right]}
\newcommand{\EE}[1]{\mathbb{E}\left[#1\right]}
\newcommand{\EEP}[1]{\mathbb{E}_P\left[#1\right]}
\newcommand{\EES}[1]{\mathbb{E}_S\left[#1\right]}
\newcommand{\EESb}[1]{\mathbb{E}_S\bigl[#1\bigr]}
\newcommand{\EESB}[1]{\mathbb{E}_S\Bigl[#1\Bigr]}
\newcommand{\EEZ}[1]{\mathbb{E}_Z\left[#1\right]}
\newcommand{\EEb}[1]{\mathbb{E}\bigl[#1\bigr]}
\newcommand{\EEtb}[1]{\mathbb{E}_t\bigl[#1\bigr]}
\newcommand{\EXP}{\mathbb{E}}
\newcommand{\EEs}[2]{\mathbb{E}_{#2}\left[#1\right]}
\newcommand{\EEu}[1]{\mathbb{E}_{\u}\left[#1\right]}
\newcommand{\EEst}[1]{\mathbb{E}_{*}\left[#1\right]}
\newcommand{\EEo}[1]{\mathbb{E}_{0}\left[#1\right]}
\newcommand{\PPt}[1]{\mathbb{P}_t\left[#1\right]}
\newcommand{\EEt}[1]{\mathbb{E}_t\left[#1\right]}
\newcommand{\PPi}[1]{\mathbb{P}_i\left[#1\right]}
\newcommand{\EEi}[1]{\mathbb{E}_i\left[#1\right]}
\newcommand{\PPc}[2]{\mathbb{P}\left[#1\left|#2\right.\right]}
\newcommand{\PPs}[2]{\mathbb{P}_{#2}\left[#1\right]}
\newcommand{\PPcs}[3]{\mathbb{P}_{#3}\left[#1\left|#2\right.\right]}
\newcommand{\PPct}[2]{\mathbb{P}_t\left[#1\left|#2\right.\right]}
\newcommand{\PPcc}[2]{\mathbb{P}\left[\left.#1\right|#2\right]}
\newcommand{\PPcct}[2]{\mathbb{P}_t\left[\left.#1\right|#2\right]}
\newcommand{\PPcci}[2]{\mathbb{P}_i\left[\left.#1\right|#2\right]}
\newcommand{\EEc}[2]{\mathbb{E}\left[#1\left|#2\right.\right]}
\newcommand{\EEcc}[2]{\mathbb{E}\left[\left.#1\right|#2\right]}
\newcommand{\EEcs}[3]{\mathbb{E}_{#3}\left[\left.#1\right|#2\right]}
\newcommand{\EEcct}[2]{\mathbb{E}_t\left[\left.#1\right|#2\right]}
\newcommand{\EEcci}[2]{\mathbb{E}_i\left[\left.#1\right|#2\right]}
\def\argmin{\mathop{\mbox{ arg\,min}}}
\def\argmax{\mathop{\mbox{ arg\,max}}}
\newcommand{\ra}{\rightarrow}

\newcommand{\siprod}[2]{\langle#1,#2\rangle}
\newcommand{\iprod}[2]{\left\langle#1,#2\right\rangle}
\newcommand{\Siprod}[2]{\left\langle#1,#2\right\rangle_{|S}}
\newcommand{\biprod}[2]{\bigl\langle#1,#2\bigr\rangle}
\newcommand{\Biprod}[2]{\Bigl\langle#1,#2\Bigr\rangle}
\newcommand{\norm}[1]{\left\|#1\right\|}
\newcommand{\bnorm}[1]{\bigl\|#1\bigr\|}
\newcommand{\onenorm}[1]{\norm{#1}_1}
\newcommand{\twonorm}[1]{\norm{#1}_2}
\newcommand{\infnorm}[1]{\norm{#1}_\infty}
\newcommand{\tvnorm}[1]{\norm{#1}_{\mathrm{TV}}}
\newcommand{\Hnorm}[1]{\norm{#1}_{\HH}}
\newcommand{\ev}[1]{\left\{#1\right\}}
\newcommand{\bev}[1]{\bigl\{#1\bigr\}}
\newcommand{\abs}[1]{\left|#1\right|}
\newcommand{\babs}[1]{\bigl|#1\bigr|}
\newcommand{\pa}[1]{\left(#1\right)}
\newcommand{\bpa}[1]{\bigl(#1\bigr)}
\newcommand{\Bpa}[1]{\Bigl(#1\Bigr)}
\newcommand{\BPA}[1]{\Biggl(#1\Biggr)}
\newcommand{\sign}{\mbox{sign}}

\newcommand{\wh}{\widehat}
\newcommand{\wt}{\widetilde}

\newcommand{\hr}{\wh{r}}
\newcommand{\hf}{\wh{f}}
\newcommand{\hs}{\wh{s}}
\newcommand{\bA}{\overline{A}}
\newcommand{\deff}{d_{\text{eff}}}

\newcommand{\Cinf}{C_\infty}
\newcommand{\tfs}{\wt{f}_{\sigma}}
\newcommand{\tQ}{\wt{Q}}
\newcommand{\tW}{\wt{W}}
\newcommand{\tR}{\wt{R}}
\newcommand{\tPhi}{\wt{\Phi}}
\newcommand{\tZ}{\wt{Z}}
\newcommand{\tS}{\wt{S}}
\newcommand{\tD}{\wt{D}}
\newcommand{\tw}{\wt{w}}
\newcommand{\bw}{\overline{w}}
\newcommand{\hX}{\wh{X}}
\newcommand{\hS}{\wh{\Sigma}}
\newcommand{\hw}{\wh{w}}
\newcommand{\hx}{\wh{\boldsymbol{x}}}
\newcommand{\bX}{\overline{X}}
\newcommand{\bY}{\overline{Y}}
\newcommand{\bS}{\overline{S}}
\newcommand{\bZ}{\overline{Z}}
\newcommand{\bz}{\overline{z}}
\newcommand{\bW}{\overline{W}}
\newcommand{\bP}{\overline{P}}
\newcommand{\bDelta}{\overline{\Delta}}
\newcommand{\Pj}{P_X\otimes P_Y}
\newcommand{\barf}{\overline{f}}
\newcommand{\hQ}{\wh{Q}}
\newcommand{\hP}{\wh{P}}
\newcommand{\tP}{\wt{P}}
\newcommand{\transpose}{^\mathsf{\scriptscriptstyle T}}

\newcommand{\tc}{\wt{c}}
\newcommand{\bloss}{\overline{\ell}}
\newcommand{\hloss}{\wh{\ell}}
\newcommand{\loss}{\ell}
\newcommand{\gen}{\textup{gen}}

\definecolor{PalePurp}{rgb}{0.66,0.57,0.66}
\definecolor{PaleBlue}{rgb}{0.33,0.33,0.66}
\newcommand{\todoG}[1]{\todo[color=PalePurp!30]{\textbf{Grg:} #1}}
\newcommand{\todoB}[1]{\todo[color=PaleBlue!30]{\textbf{Bap:} #1}}
\newcommand{\todoGI}[1]{\todo[inline,color=PalePurp!30]{#1}}
\newcommand{\redd}[1]{\textcolor{red}{#1}}
\newcommand{\orange}[1]{\textcolor{orange}{#1}}
\newcommand{\grg}[1]{\textcolor{red}{[\textbf{Grg:} #1]}}
\newcommand{\gbr}[1]{\textcolor{orange}{[\textbf{Gbr:} #1]}}

\newcommand{\regret}{\mathrm{regret}}

\newcommand{\hL}{\wh{L}}
\newcommand{\snr}{\mathtt{SNR}}

\newcommand{\alg}{\mathcal{A}}
\newcommand{\Zw}{\mathcal{Z}}
\newcommand{\bL}{\overline{L}}
\newcommand{\bH}{h}
\renewcommand{\th}{\wt{h}}
\newcommand{\law}{\text{law}}

%
\runningtitle{Online-to-PAC generalization bounds under graph-mixing dependencies}

%

\twocolumn[

\aistatstitle{Online-to-PAC generalization bounds under\\ graph-mixing dependencies}

\aistatsauthor{Baptiste Abélès \And Eugenio Clerico \And  Gergely Neu}

\aistatsaddress{Universitat Pompeu Fabra, Barcelona, Spain} ]

\begin{abstract} 
Traditional generalization results in statistical learning require a training data set made of independently drawn examples. Most of the recent efforts to relax this independence assumption have considered either purely temporal (mixing) dependencies, or graph-dependencies, where non-adjacent vertices correspond to independent random variables. Both approaches have their own limitations, the former requiring a temporal ordered structure, and the latter lacking a way to quantify the strength of inter-dependencies. In this work, we bridge these two lines of work by proposing a framework where dependencies decay with graph distance. We derive generalization bounds leveraging the online-to-PAC framework, by deriving a concentration result and introducing an online learning framework incorporating the graph structure. The resulting high-probability generalization guarantees depend on both the mixing rate and the graph's chromatic number.
\end{abstract}

\section{INTRODUCTION}\label{sec:intro}
Consider the problem of predicting house prices based on  data collected from a variety of 
locations. The value does not only depend on factors like home size, age, and amenities, but is also influenced by the neighborhood. In the language of probability theory, this can be modeled with a set of \emph{dependent} random variables, with prices of neighboring houses showing positive correlation that decays with distance. 
Similar dependencies occur between users' opinions on social networks, where connected members are more likely to share similar views \citep{montanari2010spread}. In this paper, we study the generalization ability 
of learning algorithms trained on such correlated data sets, where dependencies are encoded in graph structures.

In machine learning, a model's gap in accuracy on training data and new, previously unseen, inputs is known as \textit{generalization error}. Controlling this quantity offers theoretical guarantees on the typical performance on future data, reflecting the algorithm's ability to infer patterns \citep{shalevshwartz2014understanding}. In the past decades, a vast body of literature on this area has emerged, developing tools such as Rademacher complexity, VC dimension, uniform stability, and PAC-Bayesian inequalities \citep{vapnik2000nature, bousquet2002stability, bousquet2004introduction, alquier2024user}. Yet, the great majority of current analyses consider
training data sets  made of independent and identically distributed (i.i.d.) examples, a strong requirement  unrealistic for many applications (e.g., traffic forecasting \citep{yu2018spatio}, stock price prediction \citep{ariyo2014stock}, or the  examples above). 

Recently, interest in statistical learning frameworks accounting for data correlations has surged. A major research line models these dependencies via mixing assumptions (see \citealt{bradley2005basic} for  several  notions of mixing, such as $\alpha$-, $\beta$-, $\phi$-, and $\psi$-mixing), which control how quickly the influence between random variables decays as the (temporal, spacial,
etc.) distance between them grows. This setting provide a quantitative measure of the dependencies among the 
data points, but has the major drawback of requiring data to have a well defined ordered  structure. An alternative common framework takes a more qualitative approach, where the dependencies 
are captured by a \emph{dependency graph} that assigns an edge to any pair of vertices whose associated data are 
dependent. This approach can encode correlations among non-ordered data, but leads to loose results when the 
actual dependencies are weak \citep{janson2004large}. In this work, we propose combining the mixing and graph-based perspectives to tackle situations where the strength of the dependencies is somehow known, yet the data lack an ordered  structure. In this \emph{graph-mixing} scenario, the correlations decay as the graph distance increases.


To prove our generalization results, we follow an algorithmic approach that derives  guarantees 
for statistical learning using tools for \emph{regret analysis} in online learning. Online learning is a framework 
that deals with sequential decision problems, where a \textit{learner} (a.k.a.\ \textit{player}) interacts with 
an evolving environment. The learner's goal is to select actions over time to minimize the regret (a quantity comparing 
the player’s actions to the best fixed action in hindsight). We refer to \cite{cesabianchi2006prediction} and 
\cite{orabona2019modern} for thorough introductions to the subject. A recent line of research (commonly called 
\textit{algorithmic statistics}) has explored unconventional ways to tackle classical statistical challenges, drawing 
connections with online learning. This approach has  successfully addressed problems such as hypothesis testing 
\citep{grunwald2024safe}, decision making \citep{foster2021statistical}, mean estimation \citep{orabona2024tight}, and 
martingale concentration \citep{rakhlin2017equivalence}. This strategy has also been applied to study  generalization  in statistical learning, for instance leading to Rademacher \citep{kakade2008complexity} and PAC-Bayesian 
bounds \citep{jang2023tighter, lugosi2023online, abeles2024generalization, chatterjee2024generalization}. Most of these methods split the problem into two parts, a worst-case one that is dealt with online 
regret analysis, and a probabilistic one. When studying generalization with i.i.d.\ data (as in 
\citealt{jang2023tighter} and \citealt{lugosi2023online}), the probabilistic part reduces to upper bounding the 
deviations of a martingale. This martingale structure is lost if dependencies among the training data are 
present. \cite{abeles2024generalization} addressed this issue for  mixing processes, 
imposing a delayed-feedback constraint on the player's strategy in the framework of \citet{lugosi2023online}, which 
allowed them to decompose the overall error into the regret of a delayed online learning strategy and the fluctuations 
of a stationary mixing process. 
%
Here, we follow a similar approach to tackle more complex graph-mixing 
dependencies, by introducing a novel online learning framework on graphs, and studying the concentration of what 
we name \textit{graph-mixing processes}. 


Several works established generalization guarantees under mixing assumptions. \cite{mohri2007stability, 
mohri2010stability} obtained generalization bounds for uniformly stable algorithms under stationary $\phi$- and $\beta$-mixing, leveraging concentration tools from \cite{yu1994rates} and \cite{kontorovich2008concentration}. 
\cite{fu2023sharper} tightened these results, adapting  stability techniques from 
\cite{feldman2019high} and \cite{bousquet2020sharper} to achieve optimal rates under $\psi$-mixing assumptions. 
Stability bounds were also proved by \cite{he2016stability}, in the context of ranking, under $\phi$-mixing assumptions. 
Rademacher bounds under $\beta$-mixing stationary conditions were first established by \cite{mohri2008rademacher} via 
the blocking technique from \cite{yu1994rates}, and later extended to the non-stationary case by 
\cite{kuznetsov2017generalization}. Excess risks bounds, comparing the algorithm's output with the best predictor in 
some given class, were  established by \cite{steinwart2009fast} under geometrically $\alpha$-mixing conditions. 
Later, \cite{alquier2012model} and \cite{alquier2013prediction} employed PAC-Bayesian tools to upper bound the 
excess risk in model selection, when data are coming from a time series generalizing the standard notion of mixing 
(following ideas by \citealt{rio2000inegalites}). Excess risk was also studied by 
\cite{agarwal2013generalization}, who extended the online-to-batch conversion of \cite{cesa2001generalization} to the 
case of $\beta$- and $\phi$-mixing data, under the hypothesis of a bounded, convex, and Lipschitz loss function. 
Finally, we mention the two results that are the closest to our current work. Both \cite{chatterjee2024generalization} 
and \cite{abeles2024generalization} build on the online-to-PAC framework introduced by \cite{lugosi2023online}. The 
former follow the approach of \cite{agarwal2012generalization} to deal with $\phi$- and $\beta$-mixing stationary 
dependencies, and hence need strong regularity for the  losses. Conversely, the latter consider 
a slightly different definition of stationary mixing and (as previously mentioned) leverage the framework of online learning with delayed 
feedback to perform the online-to-PAC reduction. 

For the standard graph-dependence setting (with independent variables for non-adjacent vertices) 
a classical result is a Hoeffding-flavored concentration inequality by \cite{janson2004large}, where the graph's 
fractional chromatic number (a graph-theoretic combinatorial quantity) appears and re-weights the sample size. The core 
idea in  Janson's proof consists in splitting the graph into sets whose vertices are non-adjacent (and 
hence independent), an approach that resembles the blocking technique from \cite{yu1994rates}  for mixing 
processes. The first generalization bound in this framework was motivated by ranking (whose loss form 
naturally leads to these graph-dependencies) and obtained by \cite{usunier2005generalization}, who built 
on  \cite{janson2004large} to establish Rademacher-like bounds. Later, \cite{ralaivola2010chromatic} 
derived PAC-Bayesian bounds, again leveraging the same blocking technique. More recently, 
\cite{zhang2019mcdiarmid} proved a novel McDiarmid-type inequality for tree-dependent random variables, and 
extended it to general graphs by decomposing them into forests. Via this concentration result they established
generalization bounds for uniformly stable algorithms. We refer to \cite{zhang2024generalization} for a recent survey on 
these and other results of generalization on graphs. 

Finally, of special interest is the work of \cite{lampert2018dependency}, 
establishing concentration inequalities for the sum of random variables with graph-encoded correlations.
Their way to deal with dependencies shares many similarities with our model, 
which could actually be seen as an instance of their broader setting. However, this higher generality comes at the price 
of a rather convoluted technical analysis, introducing more complex notions of interdependence. Their approach is based 
on an approximation theorem from \cite{bradley1983approximation}  that allows to replace a set of dependent random variables with independent copies, at a 
price of an additive term involving a suitably defined separation coefficient. Another closely related approach is the 
work of \cite{feray2018weighted}, proving central limit theorems by encoding dependencies into a weighted graph, 
whose edges' weights measure the strength of the dependencies. 
In the present paper, we opted to develop a simpler framework which, while slightly less flexible, allowed us to conduct a more transparent analysis, easier to adapt to practical needs.
We defer to  future research combining our analysis with the techniques developed in these two works.




\section{THE GENERALIZATION PROBLEM}\label{sec:gen_prob}
We consider a data set $S_n = (Z_1,...,Z_n)$, drawn from a probability distribution $\mu_n$ over $\Z^n$, where $\Z$ denotes a (measurable) instance space. We  assume that each $Z_i$ has the same marginal $\mu$. The simplest situation is when all the element of $S_n$ are i.i.d., in which case $\mu_n = \mu^n$, but we will focus on more general situations. We denote as $\W$ a measurable class of hypotheses, and we let $\ell:\W\times\Z\to[0,\infty)$ be the loss function, with $\ell(w,z)$ measuring the quality of the hypothesis $w\in\W$ on the  instance $z\in\Z$. The statistical learner's goal is to find a hypothesis that performs well on average, ideally the $w\in\W$ that minimizes the \textit{population loss} $\Ell(w) = \E_\mu[\ell(w, Z)]$. Yet, $\mu$ is unknown to the learner, whose only knowledge comes from the training dataset $S_n$. We  define the \textit{empirical loss} to be the average of $\ell$ on the training data set,  $\hat\Ell_n(w) = \frac{1}{n}\sum_{t=1}^n \ell(w, Z_t)$.

A learning algorithm is a procedure to get a hypothesis $w\in\W$ starting from a training data set $S_n$. More 
generally, we will consider a \textit{randomized} learning algorithm, that is, a mapping $\A:\Z^n\to\Delta_\W$, where 
$\Delta_\W$ denotes the set of probabilities over $\W$. Note that deterministic algorithms (mapping $S_n$ to a single 
 $w$) can be seen as a particular case of the randomized setting, where the output distribution is a Dirac mass. 
As previously mentioned, the ultimate goal of the statistical learner is to optimize the population loss. In the 
context of randomized algorithms, we aim to control the expected value of this quantity. For a probability measure 
$P\in\Delta_\W$, with a slight abuse of notation, we define the \textit{expected population loss} as $\Ell(P) = \langle 
P, \Ell\rangle$, where  $\langle P, f\rangle$  denotes the expectation under $P$ of a 
measurable function $f$ on $\W$. Similarly, we define the \textit{expected empirical loss} as $\hat\Ell_n(P) = \langle 
P, \hat\Ell_n\rangle$.


For convenience, we denote the output of a randomized algorithm $\A$ as $\Pn=\A(S_n)\in\Delta_\W$. We stress here that 
$\Pn$ is a stochastic quantity, as it depends on the random training data set $S_n$. Hence, the  
expected population loss $\Ell(\Pn)$ is stochastic. We call  \textit{generalization 
bound} a high-probability inequality in the form 
\begin{equation}\label{eq:genb}\mu_n\left(\Ell(\Pn) \leq \B\big(\hat\Ell_n(\Pn), \delta\big)\right)\geq 
1-\delta\,,\end{equation}
where $\B$ is some function and $\delta\in[0,1]$ is the confidence level. For the sake of brevity, we introduce 
the notation $\pleq$ for inequalities holding with probability at least $1-\delta$, and \eqref{eq:genb} becomes $\Ell(\Pn)\pleq \B(\hat\Ell_n(\Pn), \delta)$. 


\subsection{Online-to-PAC reduction}
\cite{lugosi2023online} have recently established a framework, which they named \textit{online-to-PAC} conversion,  to obtain generalization bounds for statistical learning (in the i.i.d.\ setting) by upper bounding the regret of an online learner in the following associated online learning game.
\begin{definition}[Generalization game]
\label{def::generalization_game}
    Fix an arbitrary data set $S_n=(Z_1,\dots,Z_n)\in\Z^n$. For $n$ rounds, an online player and an adversary play the following game. At each round $t=1,2,\dots, 
n$:\begin{itemize}[topsep=-1pt, partopsep=0pt, left=0.5em]
       \setlength{\itemsep}{0pt}%
        \item the online learner picks a distribution $\pi_t \in \Delta_{\W}$; 
        \item the adversary picks a map $g_t: w \mapsto \mathcal{L}(w) - \ell(w,Z_t)$;
        \item the learner incurs a cost $-\langle \pi_t,g_t \rangle$;
        \item $Z_{t}$ is revealed to the learner.
\end{itemize}
\end{definition}
Let $\Pi = (\pi_t)_{t\geq 1}$ be an online strategy for the game above. We remark that the learner's choice of $\pi_t$ has to be done before $Z_t$ is revealed, and so can only depend on the past observations (up to round $t-1$). Fixed an arbitrary  $P\in\Delta_\W$, we define the regret of $\Pi$ against $P$ at round $n$ as $\Reg_{\Pi,n}(P) = \sum_{t=1}^n (\langle P, 
g_t\rangle-\langle \pi_t, g_t\rangle)$. The online-to-PAC reduction is the next decomposition.
\begin{theorem}[Theorem 1, \citealt{lugosi2023online}]
\label{theorem:: regret_decomposition}
Fix any  online strategy $\Pi$ for the generalization game. Any statistical learning 
algorithm $\Pn=\A(S_n)$ satisfies
\begin{equation}\label{eq::deciid}
\Ell(\Pn) \leq \hat\Ell_n(\Pn) + \frac{1}{n}\big(\Reg_{\Pi,n}(\Pn) + M_{\Pi,n}\big)\,,\end{equation}
where $M_{\Pi,n} = \sum_{t=1}^n\langle\pi_t, g_t\rangle$.
\end{theorem} 
A key remark to make use of this decomposition comes from the fact that, when the training data set $S_n$ is drawn from $\mu^n$ (and hence i.i.d.), the negation $M_{\Pi,n}$, of the
online learner's cumulative cost, is a martingale under the natural filtration induced by $S_n$, that is, the 
$\sigma$-fields $\F_t = \sigma(X_1,\dots,X_t)$. This follows from the fact that the online strategy is a predictable 
sequence of actions, as $\pi_t$ does not depend on $Z_t$ and is $\F_{t-1}$-measurable.\footnote{In general, one could let $\pi_t$  depend on other sources of randomness, not encoded in the data. This can be addressed by suitably adapting the filtration, but leaves all the results that we present unchanged.}  In particular, one can 
leverage classical martingale 
concentration results to get high probability generalization bounds in the form of \eqref{eq:genb}. We remark that in practice it is 
not necessary to actually play the generalization game. Indeed, one can  replace the regret $\Reg_{\Pi,n}$ by an 
upper bound, whenever this is known. The study and derivation of regret upper bounds is a main topic of interest in the online learning community.

For a concrete application of the above observations, we state a corollary of \Cref{theorem:: 
regret_decomposition}, which uses the parameter-free online strategy introduced by \cite{orabona2016coin} for 
learning with expert advice.
\begin{corollary}[Corollary 6, \citealt{lugosi2023online}]\label{cor:ontopac}
    Assume that $\ell$ is bounded in $[0,1]$, fix $\delta\in(0,1)$ and an arbitrary $P\in\Delta_\W$ (whose choice cannot depend on $S_n$). Then, the following generalization bound holds in high probability, uniformly for all algorithms $\Pn=\A(S_n)$,
    $$\Ell(\Pn)\pleq\hat\Ell_n(\Pn) + \sqrt{\frac{3\KL(\Pn|P)+9}{n}} + \sqrt{\frac{\log\frac{1}{\delta}}{2n}}\,.$$
\end{corollary}

We notice that the above result is in the typical form of a PAC-Bayes bound \citep{guedj2019primer, alquier2024user}, 
which typically involves a complexity term in the form of the relative entropy, $\KL$, between a data-agnostic \textit{prior} $P$ and the data-dependent \textit{posterior} $\Pn$. 
Indeed, the framework introduced by \cite{lugosi2023online} allows to recover several classic PAC-Bayesian results, and provides a range of generalizations thereof.

\subsection{Going beyond the i.i.d.\ assumption}
As it is the case for \Cref{cor:ontopac}, also the other applications of \Cref{theorem:: regret_decomposition} in \cite{lugosi2023online}  leverage the fact that $M_{\Pi,n}$ is a martingale to derive high-probability generalization bound. However, as previously mentioned, this approach cannot be directly applied when inter-dependencies among the training data are present, as these can prevent $M_{\Pi,n}$ from being a martingale. Two solutions \citep{chatterjee2024generalization, abeles2024generalization} have been recently proposed to extend the online-to-PAC reduction to situations where the correlations in the training data set can be controlled by stationary mixing assumptions. The analysis of \cite{chatterjee2024generalization} was inspired by \mbox{\cite{agarwal2013generalization}} and involves controlling the concentration properties of $M_{\Pi,n}$ under strong regularity assumptions for the loss, leaving the online formulation untouched. On the other hand, \cite{abeles2024generalization} took a perhaps more natural perspective. They introduced a delayed feedback in the online generalization game (a delay of $d$ means that $Z_t$ is only revealed at round $t+d$), ensuring that $M_{\Pi,n}$ becomes a stationary mixing process, whose concentration can be controlled via a standard blocking technique \citep{yu1994rates}. Our current work extends this approach to more general dependencies, encoded by a graph. To do so, we need to introduce a suitable online framework for learning on graphs that generalizes the online learning with delays setting. This will ensure that $M_{\Pi,n}$ is a sum of terms whose correlations can be suitably controlled, allowing to obtain high-probability generalization guarantees. These ideas will be formalized in the next section, after the definitions of several graph-theoretic concepts. 
\section{TECHNICAL TOOLS}\label{sec:graphs}

As already mentioned, we will model dependencies between random variables in the language of graphs, and will extend the online-to-PAC conversion framework of \citet{lugosi2023online} to deal with data points with a graph-dependency structure. This section presents the technical background that is necessary for formulating our assumptions on the data, and formulates an online learning framework defined on a graph structure, which will serve as basis for our reduction. 

\subsection{Basic definitions}
We first introduce here a few basic definitions related to graphs, which will be used throughout our analysis. 
\begin{definition}
    A \emph{graph} $G$ is a pair of sets $(V, E)$. The elements of $V$ are called \emph{vertices}, or \emph{nodes}, and the elements of $E$ are called \emph{edges}. Each edge is an unordered pair of elements of $V$.
\end{definition}
We will only consider loopless graphs where each edge includes two distinct vertices. 
Given a graph $G$, the set of its vertices is denoted as $V(G)$, while $E(G)$ refers to its edges. Two vertices $u$ and 
$v$ of $G$ are said to be \emph{adjacent} if $\{u,v\}$ is an edge in $E(G)$, otherwise they are called non-adjacent. 
The number of edges a vertex $v$ belongs to is called the \emph{degree} of $v$, and the \emph{degree of the graph} is 
defined as the highest degree among all its vertices. The \emph{order} of a graph is the number of its vertices. A 
sequence of edges in the form $\{v_0, v_1\}, \{v_1, v_2\}\dots \{v_{t-1}, v_t\}$ 
is called a \emph{path} of length $t$, connecting $v_0$ to $v_t$. Two vertices are \emph{connected} if there is a path 
connecting them. We define the \emph{graph distance} $\dist(u, v)$ as the length of the shortest path from $u$ to $v$. 
If $u$ and $v$ are not connected, then we let $\dist(u,v)=+\infty$.

A subset $S$ of $V(G)$ is called a \emph{stable subset} of the graph $G$ if any two vertices $u$ and $v$ in $S$ are 
non-adjacent.\footnote{Stable subsets are also known as \textit{independent subsets}. However we preferred the (also commonly used) term `stable'  to avoid confusion with probabilistic independence.} A family $\{S_k\}_k$ of stable subsets of $G$ is a \emph{stable cover} if $\cup_k S_k = V(G)$. Moreover, 
a stable cover such that all the $S_k$ are disjoint is called a \emph{stable partition} of $G$. 
    The \emph{chromatic number} $\chi$ of a graph $G$ is the cardinality of the smallest stable partition of $G$, namely 
the minimum number of stable subsets needed to form a stable partition of $G$. 

More broadly, one can consider weighted families $\{(w_k, S_k)\}_k$ of stable subsets of $G$, where the $w_k$ are 
non-negative coefficients. A \emph{stable fractional cover} is a weighted family such that $\sum_k w_k \one{v\in S_k} 
\geq 1$, for each vertex $v\in V(G)$. If $\sum_k w_k \one{v\in S_k} = 1$ for any $v$, we speak of a \emph{stable fractional 
partition}. The \emph{fractional chromatic number} $\chi_f$ of  $G$ is the minimal value of $\sum_k w_k$, among 
all the stable fractional partitions of $G$. As any stable partition is a stable fractional partition with all the 
weights set to $1$, we see that $\chi_f\leq \chi$. 

The previous definitions can be generalized by replacing the non-adjacency condition with one involving a minimal 
distance. We give formal definitions for the resulting objects, which play a key role in our analysis.
\begin{definition}A $d$-\emph{stable subset} $S$ of $G$ is a subset of $V(G)$ such that $\dist(u,v)\geq d$ for any two distinct elements 
$u$ and $v$ in $S$.
\end{definition}Note that the $2$-stable subsets of $G$ are exactly its stable subsets, while any subset of $V(G)$ is 
$1$-stable. 
\begin{definition}
    A $d$-\emph{stable fractional partition} of $G$ is a weighted family of $d$-stable subsets of $G$, $\{(w_k, S_k)\}_k$, such 
that $\sum_k w_k \one{v\in S_k} = 1$ for all $v\in V(G)$.
\end{definition}
\begin{definition}
    The \emph{fractional $d$-chromatic number} $\chif$ of $G$ is the minimal value of $\sum_k w_k$, among all the $d$-stable 
fractional partitions of $G$.
\end{definition}
Another way of thinking about $d$-stable sets is in terms of power graphs. The $d$-th power graph of $G$ is a graph 
$G^d$ such that $V(G^d) = V(G)$, with an edge for 
any two vertices whose distance (in $G$) is at most $d$. The $d$-stable subsets of $G$ are exactly the stable subsets of 
$G^{d-1}$, and therefore $\chif = \chi_f(G^{d-1})$.

\subsection[(G,phi)-mixing processes]{$(G,\phi)$-mixing processes}

We will consider a dependency structure between the training data $S_n = \ev{Z_1,\dots,Z_n}$ specified in terms of a graph $G = 
(V,E)$, with the set of nodes $V$ associated to the set of data points, and the edges $E$ describing the 
pairwise dependencies between them. The strength of the dependence between any two points $Z_i$ and $Z_j$ 
is assumed to decay with the graph distance between the corresponding nodes $v_i$ and $v_j$ in the graph, with 
the graph distance between any pair $(u,v)$ defined as the length of the shortest path between the two nodes. In order 
to define the precise dependence structure between the data points (which will be formalised in Assumption \ref{assumption :: phi_dep_graph}), we will make use of the concept of a dependence 
structure that we call a \emph{$(G, \phi)$-mixing process}, defined as follows.


\begin{definition}
\label{def::phi_mixing_graph}
Let $X_G = \{X_v\}_{v \in V(G)}$ be a family of centred random variables, labelled on a graph $G$. We say that $X_G$ is 
a $(G,\phi)$-\emph{mixing process} if there exists a non-negative non-increasing sequence $\phi=(\phi_d)_{d >0}$ such that, for 
any $v \in V$, 
\[ \EEc{X_v}{\F_{v,d}} \leq \phi_{d}\,,\]
where $\F_{v,d}=\sigma(\{X_{v'} : \dist(v,v') \geq d\})$. 
\end{definition}
When $G$ is a chain (with nodes indexed by time $t$
, and edges connecting consecutive time indices), the 
above definition is closely related to standard mixing assumptions, suggesting that the process effectively forgets 
random variables that are sufficiently far apart in time.  The two main differences are that our condition focuses on 
expectations rather than total variation distance (or alike), and, since we use undirected graphs, it does not account 
for the direction of time as in typical mixing processes.
Furthermore, the graph-dependency structure considered by \cite{janson2004large}, \cite{usunier2005generalization}, and \cite{zhang2019mcdiarmid} is recovered by letting $\phi$ be a threshold sequence, such that $\phi_d=0$ for all $d>d^\star$, and $\phi_{d}=+\infty$, for $d\leq d^\star$. In a way, the $(G,\phi)$-mixing processes capture both the qualitative aspect of the standard graph-dependence, and the quantitative side of the mixing conditions.

Intuitively, one can expect the empirical mean of $(G,\phi)$-mixing processes to concentrate around their true mean 
(zero) at a rate that is determined by the overall strength of dependencies: densely connected graphs are expected to 
yield poor concentration as compared to graphs with fewer connections. The measure of ``connectedness'' of the graph 
that we use is the fractional $d$-chromatic number $\chif$. The following proposition provides a bound on the 
empirical mean of $(G,\phi)$-mixing processes with bounded range.
\begin{proposition}
    \label{prop::graph_block_concentration_bound}
    Let $X_G$ be a $(G,\phi)$-mixing process, where $G$ is a graph of order $n$. Assume all the $X_v$ take values in a bounded interval of length $\Delta$, are centered, and have all the same marginal distribution. Then, for any $\delta > 0$, the following high probability inequality holds:
    \[\frac{1}{n}\sum_{v \in V(G)} X_v\pleq \min_{d=1\dots n}
    \left(\phi_d + \sqrt{\frac{\Delta^2\chif}{2n}\log\frac{1}{\delta}}\right)\,.\]
\end{proposition}

\subsection{Sequential learning on graphs}
\label{sec::SL_on_Graphs}
We aim to generalize the online-to-PAC approach introduced in \Cref{sec:gen_prob} to derive generalization bounds for data with a graph-dependency structure. In order to do this, we need to define a class of online learning games that respects the graph structure that underlies the data. This section presents this class of games, which we call \emph{sequential learning on graphs}.

%




Let $\A$ and $\B$ be two sets, dubbed the action space and the outcome space. We assume that $\A$ is a vector space. 
Given a graph $G$ of order $n$, for each $v\in V(G)$ we define two sets $\A_v\subseteq\A$ and $\B_v\subseteq\B$. We
assume that $\A_v$ is a convex subset of $\A$. We also define a cost function $C_v:\A\times\B\to\R$. We consider an arbitrary ordering $\{v_t\}_{t=1}^n$ of $G$, constituting a permutation of the vertices of $G$. In each round $t=1,2,\dots,n$, the player moves to node $v_t$ 
and picks an action $a_t\in \A_{v_t}$. Then, the outcome $b_t\in\B_{v_t}$ is revealed. The player incurs a cost 
$c_t(a_t, b_t)$, where $c_t= C_{v_t}$. The player can select their actions using past information only, 
namely at round $t$ the action can depend on $b_1,\dots, b_{t-1}$ and on the previous actions, but not on the present and 
future outcomes. For a fixed \emph{comparator} $a\in\A$ and a player's strategy $\Pi = (a_t)_{t\in[n]}$, we define the \emph{regret} of $\Pi$ against $a$ at round $T\leq n$ as
$$\Reg_{\Pi,T}(a) = \sum_{t=1}^T\big(c_t(a_t, b_t)-c_t(a, b_t)\big)\,.$$

In the specific game that we consider, the graph structure $G$ is used to pose further constraints on how the player is 
allowed to select their actions. In particular, we will consider \textit{sheltered} players, who are only allowed to use information from nodes that are ``sufficiently far'' from the currently selected node $v_t$. To make this formal, we define the 
\emph{$d$-exterior} of a node $v$ (where $d\in[n]$) as $U_{v, d} = \{u\in G\,:\,\dist(u,v)\geq d\}$.
\begin{definition}
In the online game defined above, a \emph{$d$-sheltered learner} is a player whose action $a_t$ in round $t$ can only depend on 
outcomes $b_s$, from rounds $s<t$ such that $v_s\in U_{v_t,d}$.
\end{definition}


The following result shows that an upper bound on the regret of a standard learner often translates into an upper bound 
for the regret of a $d$-sheltered learner. 
\begin{proposition}\label{prop:regret_distance_OL}
Assume that, for all $v\in V(G)$, the cost $C_v$ is convex in $a$. If there exists a standard online strategy $\Pi$ achieving regret $\Reg_{\Pi,T}(a) \leq F(T)$ for any $T\leq n$, where $F$ is a concave function, then, for any $d\in[n]$, there is a $d$-sheltered learner with strategy $\Pi_d$, whose 
regret is bounded as $$\Reg_{\Pi_d,n}(a) \leq \chif F\big(n/\chif \big)\,.$$
\end{proposition}
We obtain the above result in a constructive way, by explicitly devising a $d$-sheltered learner's strategy by averaging the actions of several standard players. 



Notably, the resulting class of games generalizes the well-studied setting of online learning with delayed feedback 
\citep{weinberger2002delayed,joulani2013online}. Indeed, this setting is seen as the special case where $G$ is a chain and the player is constrained to be $d$-sheltered, with $d$ corresponding to the delay in observing the feedback, and  $\chif  = d$. The rates of \cite{weinberger2002delayed} and \cite{joulani2013online} are recovered by our result. We defer a discussion of other related online  
settings to Section~\ref{sec:conclusions}.

\section{GENERALIZATION BOUNDS UNDER GRAPH-MIXING}
\label{sec::Gen_Bound}
We are now ready to state our assumptions on the dependence structure of the training data, and provide our main 
results: the graph-mixing counterparts of the generalization bounds of \Cref{sec:gen_prob}.



Our main assumption on the dependencies is that, for any hypothesis $w\in\W$, the losses $\ell(w, Z_t)$ constitute a 
$(G,\phi)$-mixing process. This is formalized as follows.
\begin{assumption}
\label{assumption :: phi_dep_graph}
    Let $S_n=(Z_1,\dots,Z_n)$ be a training data set drawn from a distribution $\mu_n$ on $\Z^n$, such that each $Z_t$ 
has the same marginal distribution $\mu$. We assume that there exists a graph $G$ (of order $n$), a bijection 
$\iota:G\to[n]$, and a non-negative non-increasing sequence $\phi = (\phi_d)_{d>0}$, such that, for all $w \in \W$, the 
graph-labelled  process $X_G(w)=\big(X_v(w)\big)_{v\in V(G)}$ is a $(G,\phi)$-mixing process, where $$X_v(w)= 
\Ell(w)-\ell(w,Z_{\iota(v)})\,.$$
\end{assumption}
This assumption is essentially an extension to the graph setting of the mixing condition proposed by 
\cite{abeles2024generalization}. It comes from the intuition that the loss associated with the observations 
$Z_v$ becomes almost independent with respect to the family of points which are at least $d$ edges away in the 
associated graph. 

We can now state the graph-mixing counterpart of \Cref{theorem:: regret_decomposition}. First, we notice that the 
generalization game of \Cref{def::generalization_game} induces an online problem on $G$.
\begin{definition}[Generalization game on $G$]\label{def::generalization_game_graph}
Consider a training data set $S_n$ satisfying Assumption \ref{assumption :: phi_dep_graph} with graph $G$ and bijection 
$\iota$. Consider the following online game on $G$. For all $v\in V(G)$, let $\A_v=\A=\Delta_\W$, $\B_v=\B$ be the 
space of measurable functions on $\W$, and $C_v(a,b) =-\langle a, b\rangle$.     For $n$ rounds, an online player and an 
adversary play the following game. At round $t=1,\dots, n$:\begin{itemize}[topsep=-1pt, partopsep=0pt, left=0.5em]
       \setlength{\itemsep}{0pt}%
        \item the online player moves on $v_t=\iota^{-1}(t)$;
        \item the online player picks a distribution $\pi_t \in \Delta_{\W}$; 
        \item the adversary picks a map $g_t: w \mapsto \mathcal{L}(w) - \ell(w,Z_t)$;
        \item the learner incurs a cost $-\langle \pi_t,g_t \rangle$;
        \item $Z_{t}$ is revealed to the learner.\end{itemize}
\end{definition}

Combining the results from Sections \ref{sec:graphs} and \ref{sec::SL_on_Graphs} we obtain the following generalization 
result.
\begin{theorem}
\label{theorem:: non-iid_gen_bound_graph}
Consider a data set $S_n$ that satisfies Assumption \ref{assumption :: phi_dep_graph}. Fix $d\in[n]$ and an arbitrary 
strategy $\Pi$ of a $d$-sheltered player for the game of \Cref{def::generalization_game_graph}. For $v\in V(G)$, 
define $\wt X_v = -\langle\pi_{\iota(v)}, g_{\iota(v)}\rangle$. Then, any statistical learning algorithm 
$\Pn=\A(S_n)$ 
satisfies
\begin{equation}\label{eq:graphontopac}
\Ell(\Pn) \leq \hat\Ell_n(\Pn) + \frac{\Reg_{\Pi,n}(\Pn) + M_{\Pi}}{n}\,,\end{equation}
with $M_{\Pi}=\sum_{v\in V(G)}\wt X_v$. Moreover, $\wt X_G$ is a $(G, \wt\phi)$-mixing process, 
where we let $\wt\phi_{d'}=\phi_{d'}$ for $d'\geq d$, and $\wt\phi_{d'}=+\infty$ for $d'<d$.
\end{theorem}
\begin{proof}
\eqref{eq:graphontopac} is equivalent to \eqref{eq::deciid}, so we will only need to show that $\wt X_G$ is 
$(G,\wt\phi)$-mixing. Clearly, when $d'<d$, we have $\E[X_v|\F_{v,d'}]\leq \wt\phi_{d'}=+\infty$. For $d'\geq d$, 
$\F_{v,d'}\supseteq\F_{v,d}$, and so $\pi_{\iota(v)}$ is $\F_{v,d'}$-measurable, by definition of $d$-sheltered learner. 
Hence, $\E[X_v|\F_{v,d'}] = -\langle\pi_{\iota(v)}, \E[g_{\iota(v)}|\F_{v,d'}]\rangle$, and 
$\E[g_{\iota(v)}|\F_{v,d'}]\leq \phi_{d'} = \wt\phi_{d'}$ by Assumption \ref{assumption :: phi_dep_graph}.
\end{proof}
The usefulness of the above result comes from the fact that we know how to upper bound (in high probability) $(G, 
\phi)$-mixing processes. Hence, we can derive a graph-mixing counterpart of \Cref{cor:ontopac}.
\begin{corollary}
\label{cor:mixing_ontopac}
    Consider a data set $S_n$ that satisfies Assumption \ref{assumption :: phi_dep_graph}, assume that $\ell$ is bounded 
in $[0,1]$, fix $\delta\in(0,1)$, $d\in[n]$, and an arbitrary $P\in\Delta_\W$ (whose choice cannot depend on $S_n$). 
Then, the following generalization bound holds in high probability, uniformly for all algorithms $\Pn=\A(S_n)$,
    \begin{align*}\Ell(\Pn)&\pleq\hat\Ell_n(\Pn)+\phi_d  \\ &+\left(\sqrt{3\KL(\Pn|P)+9} + 
\sqrt{\frac{1}{2}\log\frac{1}{\delta}}\right)\sqrt{\frac{\chif}{n}}\,.\end{align*}
\end{corollary}
\begin{proof}
The proof  combines \Cref{theorem:: non-iid_gen_bound_graph}, \Cref{prop::graph_block_concentration_bound}, and 
\Cref{prop:regret_distance_OL}. Fix $P\in\Delta_\W$ and $d\in[n]$. By a slight generalization of Corollary 6 in 
\cite{orabona2016coin} (see the proof of Corollary 6 in \citealp{lugosi2023online}), we know that for any $d$ there is a 
(standard) online strategy for the game of \Cref{def::generalization_game_graph}, whose regret (for any comparator $P'$) 
is upper bounded by $\sqrt{(3+\KL(P'|P))3n}$. By \Cref{prop:regret_distance_OL}, there is a $d$-sheltered online strategy 
$\Pi_d$ whose regret is upper bounded by $\sqrt{\scriptstyle(3+\KL(P'|P))3n\chif}$. Apply 
\Cref{theorem:: non-iid_gen_bound_graph}. Since $\wt \phi_d=\phi_d$, by \Cref{prop::graph_block_concentration_bound} 
$\frac{1}{n}M_{\Pi_n^{(d)}}\pleq \phi_d + \sqrt{\scriptstyle\frac{1}{2n}\chif\log\frac{1}{\delta}}$, and so we conclude.
\end{proof}
We stress that this is only one of the many possible bounds that can be derived from our framework, given that 
different online learning algorithms may lead to different regret bounds. We refer to Section~3 of 
\citet{lugosi2023online} for further examples, including generalized PAC-Bayesian bounds where the relative entropy, $\KL$, 
appearing in the above bound is replaced by other strongly convex functionals of $\Pn$.

The tightness of these bounds relies on the chromatic number of the power graph, and the 
coefficients $\phi$ characterizing the strength of dependencies. Typical assumptions regarding the latter include 
functions of the form $\phi_d = C e^{-d/\tau}$, for some $C, \tau > 0$ (called \emph{geometric mixing}), or $\phi_d = C 
d^{-r}$ for some $C, r > 0$ (called \emph{algebraic mixing}). As for the chromatic indices, it is known that they can 
always be bounded as $\chi_f^{(d)} =  \OO(\Delta^d)$ where $\Delta$ is the degree of the original graph. It is often 
possible, though, to show tighter bounds for graphs that arise naturally in practical applications. We demonstrate a few 
concrete examples below, and refer to \cite{alon2002chromatic} for a more exhaustive treatment.

\paragraph{Temporal processes.} The simplest non-trivial example is the class of mixing processes in time, which we have 
already mentioned extensively. 
These processes can be modeled by a graph $G$, whose nodes 
correspond to the time indices ${1,2,\dots,n}$, and edges connect neighboring indices, namely $E(G) =(\{(t,t+1\})_{t\in[n]}$. This 
can model a variety temporally-dependent data sequences, such as stock prices, energy consumption, or sensor data from 
physical environments (see, e.g., \citealp{ariyo2014stock, takeda2016using}). In this case, one can easily see that $\chi_f^{(d)} = d$. Thus, in this setting our guarantees almost exactly recover the recent results of 
\citet{abeles2024generalization}. We refer the reader to their work for details.

\paragraph{Processes on a spatial grid.} A direct generalization of the previous case is where the graph is a 
2-dimensional grid of size $n = I\times J$, for some integers $I$ and $J$. Such graphs can model spatially organized data like the house-price example mentioned in the 
introduction. A straightforward calculation shows that $\chi_f^{(d)}$ is of order $d^2$ for this class of 
graphs\footnote{To see this, note that the set of nodes reachable through a path of length $d$ roughly corresponds to 
the nodes falling into a square of diagonal $2d$ on the two-dimensional plane.}. For the sake of concreteness, let us suppose 
that the mixing is geometric. Then, the guarantee of \Cref{cor:mixing_ontopac} implies a generalization bound of order
${\scriptstyle
    \mathcal{O}\pa{Ce^{-d/\tau}  + \sqrt{\frac{d^2}{n}\pa{\KL(\Pn|P) + 
 \log\frac{1}{\delta}}}}
}$.
Setting $d = \tau \log (Cn)$, neglecting log factors this becomes
$   {\scriptstyle \wt{\mathcal{O}}\Bpa{\sqrt{{\tau^2 \pa{\KL(\Pn|P) + 
 \log\frac{1}{\delta}}/n}}}}$,
thus demonstrating a linear dependence with the mixing-time parameter $\tau$. This argument can be easily extended 
to other planar graphs of similar regularity, and generalized to $k$-dimensional grids where the chromatic indices scale as $d^k$, 
eventually yielding a dependence of order $\tau^{k/2}$ on the mixing time.

\section{CONCLUSION}\label{sec:conclusions}
We have introduced a new model for statistical learning with dependent data, and provided a general framework for 
developing generalization bounds for learning algorithms. A key tool in our analysis was a reduction to a 
family of online learning games. We conclude by discussing some further related work and highlighting some 
interesting open problems.

\paragraph{The tightness of our bounds.} Our upper bounds on the generalization error depend on variations of the 
chromatic number of the dependency graph. While it is easy to construct hard examples where this dependence is tight 
 (e.g., when $G$ is composed of several disconnected cliques), 
it is not clear if our bounds can be further improved to scale with more fine-grained graph properties. On a related 
note, it is also easy to construct examples where our bounds are vacuous, yet it still should be possible to estimate 
the test error with good rates. To this end, consider a graph of size $n$, with $n/2$ isolated vertices and the 
remaining $n/2$ vertices forming a clique. The chromatic number of this graph is $n/2$, which makes our bounds 
trivial. However, in such a case it is clearly a bad idea to measure the training error on all samples: the 
heavy dependence of the second half of the data points leads to a massive bias. This bias, however, can be completely 
removed by simply discarding the second half of observations and only using the i.i.d.~samples. This pathological case 
suggests that the empirical mean can be an arbitrarily poor estimator of the mean, and much more efficient estimators 
can be constructed by taking the graph structure into account. Our analysis suggests an obvious way to 
do so: find the largest $d$-stable subset and then use only data points from this set. Our techniques can be used to 
show the same generalization bound for this method as for the empirical mean, but the example above indicates that 
its actual performance could be much better. The downside, of course, is that this approach requires full knowledge of 
the graph and requires additional computation. In contrast, our bounds need only high-level information about the graph, as they only assume knowledge of the chromatic numbers, which might easier to estimate than finding stable sets.
We leave a detailed investigation of this interesting question open for future work.

\paragraph{Online learning on graphs.} To our knowledge, the sequential learning framework we introduce in 
\Cref{sec::SL_on_Graphs} has not appeared in the previous literature. That said, several similar models have been 
studied. The works of \citet{guillory09label} and\cite{cesabianchi10} consider learning labelings on graphs via actively querying a 
subset of the labels, and provide mistake bounds that depend on a joint notion of complexity of the labeling and the 
graph. Obtaining guarantees in terms of such problem-dependent notions of complexity would be desirable in our setting 
as well, but unfortunately their model is rather different from ours. A more relevant setting is the one studied by 
\citet{cesabianchi20cooperative}, who study an online learning protocol defined on a network of agents. In each round, one agent wakes 
up, needs to make a prediction, suffers a loss, and shares the observation with its neighbors. In a certain sense, this 
problem is the dual of ours: in our setting, a sheltered online learner is not allowed to use information from the 
neighbors of the currently active node, whereas their setting only allows using information from neighboring nodes. The 
two settings can be transformed into each other by taking the complements of the underlying graphs. Applying their 
algorithm to our setting in the most straightforward way yields guarantees that can be recovered by 
\Cref{prop:regret_distance_OL}. We find it plausible that approaching our problem from this alternative direction may 
lead to improved data-dependent guarantees (as suggested by existing follow-up work like that of \citealt{achddou2024multitask}), 
but so far we do not see sufficient evidence to prefer this rather roundabout route over our rather simple 
formulation that addresses our overall problem in more natural terms. We remain optimistic nevertheless that further 
progress on online learning with graph structures will enable improvements in the statistical learning setting we study 
in the present paper. As a final remark, the online framework that we introduced, and the way we developed to couple it with graph-mixing processes' concentration, are likely to be useful to adapt other algorithmic statistics approaches to graph-mixing dependent settings.

\subsubsection*{Acknowledgements}
The authors would like to thank Rui-Ray Zhang and G\'abor Lugosi for the insightful discussions that inspired this work. This project has received funding from the European Research Council
(ERC), under the European Union’s Horizon 2020 research
and innovation programme (Grant agreement No.\ 950180).
\bibliographystyle{apalike}
\bibliography{ref}

\begin{thebibliography}{}

\bibitem[Abeles et~al., 2024]{abeles2024generalization}
Abeles, B., Clerico, E., and Neu, G. (2024).
\newblock Generalization bounds for mixing processes via delayed online-to-pac conversions.
\newblock {\em arXiv:2406.12600}.

\bibitem[Achddou et~al., 2024]{achddou2024multitask}
Achddou, J., Cesa-Bianchi, N., and Laforgue, P. (2024).
\newblock Multitask online learning: Listen to the neighborhood buzz.
\newblock {\em AISTATS}.

\bibitem[Agarwal and Duchi, 2012]{agarwal2012generalization}
Agarwal, A. and Duchi, J.~C. (2012).
\newblock The generalization ability of online algorithms for dependent data.
\newblock {\em IEEE Transactions on Information Theory}, 59(1).

\bibitem[Agarwal and Duchi, 2013]{agarwal2013generalization}
Agarwal, A. and Duchi, J.~C. (2013).
\newblock The generalization ability of online algorithms for dependent data.
\newblock {\em IEEE Transactions on Information Theory}, 59(1):573--587.

\bibitem[Alon and Mohar, 2002]{alon2002chromatic}
Alon, N. and Mohar, B. (2002).
\newblock The chromatic number of graph powers.
\newblock {\em Combinatorics, Probability and Computing}, 11(1):1--10.

\bibitem[Alquier, 2024]{alquier2024user}
Alquier, P. (2024).
\newblock User-friendly introduction to pac-bayes bounds.
\newblock {\em Foundations and Trends in Machine Learning}, 17(2):174--303.

\bibitem[Alquier et~al., 2013]{alquier2013prediction}
Alquier, P., Li, X., and Wintenberger, O. (2013).
\newblock Prediction of time series by statistical learning: General losses and fast rates.
\newblock {\em Dependence Modeling}, 1:65--93.

\bibitem[Alquier and Wintenberger, 2012]{alquier2012model}
Alquier, P. and Wintenberger, O. (2012).
\newblock Model selection for weakly dependent time series forecasting.
\newblock {\em Electronic Journal of Statistics}, 6:1447--1464.

\bibitem[Ariyo et~al., 2014]{ariyo2014stock}
Ariyo, A.~A., Adewumi, A.~O., and Ayo, C.~K. (2014).
\newblock Stock price prediction using the arima model.
\newblock {\em UKSim-AMSS International Conference on Computer Modelling and Simulation}.

\bibitem[Bousquet et~al., 2004]{bousquet2004introduction}
Bousquet, O., Boucheron, S., and Lugosi, G. (2004).
\newblock {\em Introduction to Statistical Learning Theory}.
\newblock Springer.

\bibitem[Bousquet and Elisseeff, 2002]{bousquet2002stability}
Bousquet, O. and Elisseeff, A. (2002).
\newblock Stability and generalization.
\newblock {\em Journal of Machine Learning Research}, 2.

\bibitem[Bousquet et~al., 2020]{bousquet2020sharper}
Bousquet, O., Klochkov, Y., and Zhivotovskiy, N. (2020).
\newblock Sharper bounds for uniformly stable algorithms.
\newblock {\em COLT}.

\bibitem[Bradley, 1983]{bradley1983approximation}
Bradley, R.~C. (1983).
\newblock Approximation theorems for strongly mixing random variables.
\newblock {\em Michigan Mathematical Journal}, 30(1):69--81.

\bibitem[Bradley, 2005]{bradley2005basic}
Bradley, R.~C. (2005).
\newblock Basic properties of strong mixing conditions: A survey and some open questions.
\newblock {\em Probability Surveys}, 2:107--144.

\bibitem[Cesa-Bianchi et~al., 2020]{cesabianchi20cooperative}
Cesa-Bianchi, N., Cesari, T., and Monteleoni, C. (2020).
\newblock Cooperative online learning: Keeping your neighbors updated.
\newblock {\em ALT}.

\bibitem[Cesa-Bianchi et~al., 2001]{cesa2001generalization}
Cesa-Bianchi, N., Conconi, A., and Gentile, C. (2001).
\newblock On the generalization ability of on-line learning algorithms.
\newblock {\em NeurIPS}.

\bibitem[Cesa~Bianchi et~al., 2010]{cesabianchi10}
Cesa~Bianchi, N., Gentile, C., Vitale, F., and Zappella, G. (2010).
\newblock Active learning on trees and graphs.
\newblock {\em COLT}.

\bibitem[Cesa-Bianchi and Lugosi, 2006]{cesabianchi2006prediction}
Cesa-Bianchi, N. and Lugosi, G. (2006).
\newblock {\em Prediction, learning, and games}.
\newblock Cambridge university press.

\bibitem[Chatterjee et~al., 2024]{chatterjee2024generalization}
Chatterjee, S., Mukherjee, M., and Sethi, A. (2024).
\newblock Generalization bounds for dependent data using online-to-batch conversion.
\newblock {\em NeurIPS}.

\bibitem[Feldman and Vondrak, 2019]{feldman2019high}
Feldman, V. and Vondrak, J. (2019).
\newblock High probability generalization bounds for uniformly stable algorithms with nearly optimal rate.
\newblock {\em COLT}.

\bibitem[Foster et~al., 2021]{foster2021statistical}
Foster, D.~J., Kakade, S.~M., Qian, J., and Rakhlin, A. (2021).
\newblock The statistical complexity of interactive decision making.
\newblock {\em arXiv:2112.13487}.

\bibitem[Fu et~al., 2023]{fu2023sharper}
Fu, S., Lei, Y., Cao, Q., Tian, X., and Tao, D. (2023).
\newblock Sharper bounds for uniformly stable algorithms with stationary mixing process.
\newblock {\em ICLR}.

\bibitem[Féray, 2018]{feray2018weighted}
Féray, V. (2018).
\newblock Weighted dependency graphs.
\newblock {\em Electronic Journal of Probability}, 23:1--65.

\bibitem[Grünwald et~al., 2024]{grunwald2024safe}
Grünwald, P., de~Heide, R., and Koolen, W. (2024).
\newblock Safe testing.
\newblock {\em Journal of the Royal Statistical Society Series B: Statistical Methodology}.

\bibitem[Guedj, 2019]{guedj2019primer}
Guedj, B. (2019).
\newblock A primer on pac-bayesian learning.
\newblock {\em arXiv:1901.05353}.

\bibitem[Guillory and Bilmes, 2009]{guillory09label}
Guillory, A. and Bilmes, J.~A. (2009).
\newblock Label selection on graphs.
\newblock {\em NeurIPS}.

\bibitem[He et~al., 2016]{he2016stability}
He, F., Zuo, L., and Chen, H. (2016).
\newblock Stability analysis for ranking with stationary $\phi$-mixing samples.
\newblock {\em Neurocomputing}, 171:1556--1562.

\bibitem[Jang et~al., 2023]{jang2023tighter}
Jang, K., Jun, K.-S., Kuzborskij, I., and Orabona, F. (2023).
\newblock Tighter {PAC-B}ayes bounds through coin-betting.
\newblock {\em COLT}.

\bibitem[Janson, 2004]{janson2004large}
Janson, S. (2004).
\newblock Large deviations for sums of partly dependent random variables.
\newblock {\em Random Structures \& Algorithms}, 24(3):234--248.

\bibitem[Joulani et~al., 2013]{joulani2013online}
Joulani, P., Gy{\"o}rgy, A., and Szepesv{\'a}ri, C. (2013).
\newblock Online learning under delayed feedback.
\newblock {\em ICML}.

\bibitem[Kakade et~al., 2008]{kakade2008complexity}
Kakade, S.~M., Sridharan, K., and Tewari, A. (2008).
\newblock On the complexity of linear prediction: Risk bounds, margin bounds, and regularization.
\newblock {\em NeurIPS}.

\bibitem[Kontorovich and Ramanan, 2008]{kontorovich2008concentration}
Kontorovich, L. and Ramanan, K. (2008).
\newblock Concentration inequalities for dependent random variables via the martingale method.
\newblock {\em The Annals of Probability}, 36(6):2126--2158.

\bibitem[Kuznetsov and Mohri, 2017]{kuznetsov2017generalization}
Kuznetsov, V. and Mohri, M. (2017).
\newblock Generalization bounds for non-stationary mixing processes.
\newblock {\em Machine Learning}, 106(1):93--117.

\bibitem[Lampert et~al., 2018]{lampert2018dependency}
Lampert, C.~H., Ralaivola, L., and Zimin, A. (2018).
\newblock Dependency-dependent bounds for sums of dependent random variables.
\newblock {\em arXiv:1811.01404}.

\bibitem[Lugosi and Neu, 2023]{lugosi2023online}
Lugosi, G. and Neu, G. (2023).
\newblock Online-to-{PAC} conversions: Generalization bounds via regret analysis.
\newblock {\em arXiv:2305.19674}.

\bibitem[Mohri and Rostamizadeh, 2007]{mohri2007stability}
Mohri, M. and Rostamizadeh, A. (2007).
\newblock Stability bounds for non-i.i.d. processes.
\newblock {\em NeurIPS}.

\bibitem[Mohri and Rostamizadeh, 2008]{mohri2008rademacher}
Mohri, M. and Rostamizadeh, A. (2008).
\newblock Rademacher complexity bounds for non-i.i.d. processes.
\newblock {\em NeurIPS}.

\bibitem[Mohri and Rostamizadeh, 2010]{mohri2010stability}
Mohri, M. and Rostamizadeh, A. (2010).
\newblock Stability bounds for stationary $\phi$-mixing and $\beta$-mixing processes.
\newblock {\em Journal of Machine Learning Research}, 11(Feb):798--814.

\bibitem[Montanari and Saberi, 2010]{montanari2010spread}
Montanari, A. and Saberi, A. (2010).
\newblock The spread of innovations in social networks.
\newblock {\em Proceedings of the National Academy of Sciences}, 107(47):20196--20201.

\bibitem[Orabona, 2019]{orabona2019modern}
Orabona, F. (2019).
\newblock A modern introduction to online learning.
\newblock {\em arXiv:1912.13213}.

\bibitem[Orabona and Jun, 2024]{orabona2024tight}
Orabona, F. and Jun, K.-S. (2024).
\newblock Tight concentrations and confidence sequences from the regret of universal portfolio.
\newblock {\em IEEE Transactions on Information Theory}, 70(1):436--455.

\bibitem[Orabona and P{\'a}l, 2016]{orabona2016coin}
Orabona, F. and P{\'a}l, D. (2016).
\newblock Coin betting and parameter-free online learning.
\newblock {\em NeurIPS}.

\bibitem[Rakhlin and Sridharan, 2017]{rakhlin2017equivalence}
Rakhlin, A. and Sridharan, K. (2017).
\newblock On equivalence of martingale tail bounds and deterministic regret inequalities.
\newblock {\em COLT}.

\bibitem[Ralaivola et~al., 2010]{ralaivola2010chromatic}
Ralaivola, L., Szafranski, M., and Stempfel, G. (2010).
\newblock Chromatic {PAC-B}ayes bounds for non-iid data: Applications to ranking and stationary $\beta$-mixing processes.
\newblock {\em Journal of Machine Learning Research}, 11(Sep):1927--1956.

\bibitem[Rio, 2000]{rio2000inegalites}
Rio, E. (2000).
\newblock In{\'e}galit{\'e}s de {H}oeffding pour les fonctions lipschitziennes de suites d{\'e}pendantes.
\newblock {\em Comptes Rendus de l'Acad{\'e}mie des Sciences, S{\'e}rie I, Math{\'e}matiques}, 330:905--908.

\bibitem[Shalev-Shwartz and Ben-David, 2014]{shalevshwartz2014understanding}
Shalev-Shwartz, S. and Ben-David, S. (2014).
\newblock {\em Understanding Machine Learning - From Theory to Algorithms}.
\newblock Cambridge University Press.

\bibitem[Steinwart and Christmann, 2009]{steinwart2009fast}
Steinwart, I. and Christmann, A. (2009).
\newblock Fast learning from non-i.i.d. observations.
\newblock {\em NeurIPS}.

\bibitem[Takeda et~al., 2016]{takeda2016using}
Takeda, H., Tamura, Y., and Sato, S. (2016).
\newblock Using the ensemble {K}alman filter for electricity load forecasting and analysis.
\newblock {\em Energy}, 104.

\bibitem[Usunier et~al., 2005]{usunier2005generalization}
Usunier, N., Amini, M.~R., and Gallinari, P. (2005).
\newblock Generalization error bounds for classifiers trained with interdependent data.
\newblock {\em NeurIPS}.

\bibitem[Vapnik, 2000]{vapnik2000nature}
Vapnik, V. (2000).
\newblock {\em The nature of statistical learning theory}.
\newblock Springer-Verlag.

\bibitem[Weinberger and Ordentlich, 2002]{weinberger2002delayed}
Weinberger, M. and Ordentlich, E. (2002).
\newblock On delayed prediction of individual sequences.
\newblock {\em IEEE Transactions on Information Theory}, 48(7).

\bibitem[Yu, 1994]{yu1994rates}
Yu, B. (1994).
\newblock Rates of convergence for empirical processes of stationary mixing sequences.
\newblock {\em The Annals of Probability}, 22(1):94--116.

\bibitem[Yu et~al., 2018]{yu2018spatio}
Yu, B., Yin, H., and Zhu, Z. (2018).
\newblock Spatio-temporal graph convolutional networks: A deep learning framework for traffic forecasting.
\newblock {\em IJCAI}.

\bibitem[Zhang and Amini, 2024]{zhang2024generalization}
Zhang, R.~R. and Amini, M.~R. (2024).
\newblock Generalization bounds for learning under graph-dependence: A survey.
\newblock {\em Machine Learning}, 113(7):3929--3959.

\bibitem[Zhang et~al., 2019]{zhang2019mcdiarmid}
Zhang, R.~R., Liu, X., Wang, Y., and Wang, L. (2019).
\newblock Mcdiarmid-type inequalities for graph-dependent variables and stability bounds.
\newblock {\em NeurIPS}.

\end{thebibliography}

\newpage
\onecolumn
\section*{Supplementary material}
\appendix
\section{Omitted proofs}

\subsection[Proof of Proposition 1]{Proof of \Cref{prop::graph_block_concentration_bound}}\label{app:graph_block_concentration_bound}
The proof leverages the approach introduced by \cite{janson2004large}. Fix $d$ and consider a $d$-stable fractional partition $\{(w_k, S_k)\}_k$ of $G$. We can write $$\sum_{v\in V(G)} X_v= \sum_{v\in V(G)} X_v \sum_{k:v\in S_k}w_k = \sum_kw_k\sum_{v\in S_k}X_v\,.$$ In particular, for any $\lambda>0$, we have
\begin{equation}\label{eq:logJensen}\log \E\left[e^{\frac{\lambda}{n}\sum_{v\in V(G)}X_v}\right] = \log \E\left[e^{\frac{\lambda}{n}\sum_k w_k \sum_{v\in S_k}X_v}\right] \leq \sum_k p_k \log\E\left[e^{\frac{\lambda}{n} \frac{w_k}{p_k} \sum_{v\in S_k}X_v}\right]\,,\end{equation}
where $p$ is a probability vector ($\sum_k p_k = 1$ with $p_k> 0$ for all $k$), and in the last step we have applied Jensen's inequality, since $f\mapsto \log\E[e^f]$ is a convex mapping. 

Now, for any $k$ we can label arbitrarily the elements in $S_k$ as $v_1^{(k)},\dots, v_{n_k}^{(k)}$, where $n_k$ is the 
cardinality of $S_k$. Let us denote as $\F_i^{(k)}$ the sigma algebra $\sigma(\{X_{v_j^{(k)}}\,:\,j\leq i\})$. Since 
$S_k$ is a $d$-stable subset of $g$, recalling the notation introduced in \Cref{def::phi_mixing_graph}, we have that 
$\F_{i-1}^{(k)} \subseteq \F_{v_{i}^{(k)}, d}$. In particular, the fact that $X_G$ is a $(G,\phi)$-mixing process 
implies that $$\E[X_{v_i^{(k)}}|\F_{i-1}^{(k)}] = \E\big[\E[X_{v_i^{(k)}}|\F_{v_{i}^{(k)}, d}]\big|\F_{i-1}^{(k)}\big] 
\leq \phi_d$$ by the tower  property of conditional expectation.
Now, this implies that for any $i\leq n_k$ we have
$$\E\left[e^{\frac{\lambda}{n}\frac{w_k}{p_k}\sum_{j=1}^{i}X_{v^{(k)}_j}}\middle|\F_{i-1}^{(k)}\right] \leq \E\left[e^{\frac{\lambda}{n}\frac{w_k}{p_k}\sum_{j=1}^{i-1}X_{v_j^{(k)}}}\right] \E\left[e^{\frac{\lambda}{n}\frac{w_k}{p_k}(X_{v^{(k)}_i}-\E[X_{v^{(k)}_i}|\F_{i-1}])}\middle|\F_{i-1}^{(k)}\right]\exp\left(\frac{\lambda\phi_d}{n}\frac{w_k}{p_k}\right)\,.$$
Moreover, the fact that each $X_v$ is bounded in an interval $I$ of length $\Delta$ implies that it is $\Delta^2/4$-subgaussian with respect to any measure, and hence $$\E\left[e^{\frac{\lambda}{n}\frac{w_k}{p_k}\left(X_{v^{(k)}_i}-\E[X_{v^{(k)}_i}|\F_{i-1}]\right)}\middle|\F_{i-1}^{(k)}\right]\leq \exp\left(\frac{\lambda^2\Delta^2}{8n^2}\frac{w_k^2}{p_k^2}\right)\,.$$ Applying these arguments recursively $n_k$ times we obtain
$$\log \E\left[e^{\frac{\lambda}{n}\frac{w_k}{p_k}\sum_{v\in S_k}X_v}\right] \leq n_k\left(\frac{\lambda^2\Delta^2}{8n^2}\frac{w_k^2}{p_k^2} + \frac{\lambda\phi_d}{n}\frac{w_k}{p_k}\right)\,.$$

We can hence rewrite \eqref{eq:logJensen} as
$$\log \E\left[e^{\frac{\lambda}{n}\sum_{v\in V(G)}X_v}\right] \leq \sum_k n_k\left(\frac{\lambda^2\Delta^2}{8n^2}\frac{w_k^2}{p_k} + \frac{\lambda\phi_d}{n}w_k\right) = \sum_k\frac{\lambda^2\Delta^2}{8n^2}\frac{n_kw_k^2}{p_k} + \lambda \phi_d\,,$$
where in the last equality we used that \begin{equation}\label{eq:sumNk}\sum_k w_k n_k = \sum_k w_k \sum_{v\in S_k}1 = \sum_k\sum_{v\in V(G)}w_k\one{v\in S_k}= \sum_{v\in V(G)}\sum_{k:v\in S_k}w_k = \sum_{v\in V(G)}1 = n\,.\end{equation} 
We can now optimize the choice of $p$, by setting $p_k = \frac{w_k\sqrt{n_k}}{\sum_{k'}w_{k'}\sqrt{N_{k'}}}$. With this choice we have
$$\log \E\left[e^{\frac{\lambda}{n}\sum_{v\in V(G)}X_v}\right] \leq \frac{\lambda^2}{2n^2}\left(\sum_k w_k\sqrt{n_k}\right)^2 + \lambda\phi_d\,.$$
By Cauchy-Schwarz inequality we have
\begin{equation*}
\sum_k w_k\sqrt{n_k} = \sum_k \sqrt{w_k}\sqrt{w_kn_k} \leq \sqrt{\sum_k w_k}\sqrt{\sum_kw_kn_k} = \sqrt{n\sum_kw_k}\,,\end{equation*} where again we used \eqref{eq:sumNk}. Since the choice of the $d$-stable fractional partition is arbitrary, we can chose an optimal one, such that $\sum_kw_k = \chif$. In particular, we get
$$\log \E\left[e^{\frac{\lambda}{n}\sum_{v\in V(G)}X_v}\right] \leq \frac{\lambda^2\Delta^2\chif}{8n} + \lambda\phi_d\,.$$

By Markov's inequality, we have that for any $t>\phi_d$
$$\Pp\left(\frac{1}{n}\sum_{v\in V(G)}X_v\geq t\right)\leq \inf_{\lambda>0}\frac{\E\left[e^{\frac{\lambda}{n}\sum_{v\in V(G)}X_v}\right]}{e^{\lambda t}} \leq \inf_{\lambda>0}\exp\left(\frac{\lambda^2\Delta^2\chif}{8n}-\lambda(t-\phi_d)\right) = \exp\left(-\frac{2n}{\chif}\frac{(t-\phi_d)^2}{\Delta^2}\right)\,.$$

The conclusion now follows by setting the RHS above equal to $\delta$ and solving for $t$.\qed

\subsection[Proof of Proposition ?]{Proof of \Cref{prop:regret_distance_OL}}\label{app:proof_regret}
First, let us fix an arbitrary $d$-stable fractional partition $\{(w_k, S_k)\}_k$ of $G$. The idea is that we will run an independent player on each $S_k$. Each of them will also be a $d$-sheltered learner, as, by definition, in a $d$-stable fractional partition any two distinct vertices are distant at least $d$ from each other. We will see that, by carefully averaging the actions of these players, it is possible to obtain a $d$-sheltered learner on the full graph $G$, whose regret can be upper bounded as desired.

First, note that the ordering of $V(G)$ induces an ordering on $S_k$, and we will write $S_k = (v^{(k)}_1,\dots, v^{(k)}_{n_k})$, where $n_k$ is the cardinality of $S_k$. We now introduce some notation which will be helpful for what follows. Any vertex $v\in V(G)$ corresponds to an element in the ordered sequence $v_1,\dots, v_n$. We denote as $\iota(v)$ the index of this element (so that $\iota(v_t) = t$ for all $t$). Similarly, given an element $v\in S_k$, we denote as $\iota_k$ its index in the sequence $(v^{(k)}_1,\dots, v^{(k)}_{n_k})$.

For each $S_k$, we let run an independent copy of the standard player, and we denote their strategy as $\Pi_k = (a^{(k)}_1,\dots, a^{(k)}_{n_k})$. We assume that although these players' choices of the action are independent from each other, for each vertex each player who passes through it receives the same outcome,\footnote{This is somehow limiting the power of a potential adversary for each of these games, but this does not affect the regret bounds that hold for any possible outcome sequence. Also, notice that for the $d$-sheltered game a potential adversary is allowed to choose freely for any vertex $v$, and indeed it is this chosen outcome that each of the players of the $d$-stable subsets will see when passing on $v$.} which corresponds to the outcome that the $d$-sheltered learner running on the full graph sees. By assumption, we can choose the strategy of the standard player so that we can upper bound the regret of each $\Pi_k$ as $$\Reg_{\Pi_k,n_k}(a) = \sum_{t=1}^{n_k} \big(c_t^{(k)}(a^{(k)}_t, b_t^{(k)}) - c_t^{(k)}(a , b_t^{(k)})\big)\leq F(n_k)\,,$$ where $c_t^{(k)}=C_{v^{(k)}_t}$ and $b_t^{(k)} = b_\iota(v_t^{(k)})$ is the outcome on $v^{(k)}_t$ (which, as previously stated, only depends on the vertex and not on which player is observing it, as it corresponds to the outcome that the $d$-sheltered learner sees at round $\iota(v_t^{(k)})$). 

We will now define the $d$-sheltered learner's strategy $\Pi=(a_1,\dots,a_n)$ on the full graph. For any $v\in V(G)$, let $\kappa(v) = \{k\,:\, v\in S_k\}$. We set $a_t = A(v_t)$, with $$A(v) = \sum_{k\in \kappa(v)} w_k a_{\iota_k(v)}^{(k)}\,.$$ 
First, we notice that $A(v)\in\A_v$, since $\A_v$ is assumed to be convex and $A(v)$ is a convex mixture of elements in $\A_v$ (note that $\sum_{k\in\kappa(v)}w_k=1$ by definition of $d$-stable fractional partition). 

We now show that the above definition of $a_t$ define an admissible strategy for a $d$-sheltered learner. First, we see 
that $a_t$ only depends on what has happened on the set $H_t = \{v_s^{(k)}\,:\,k\in\kappa(v_t)\text{ and } s\leq 
\iota_k(v_t)-1\}$. Clearly, $H_t\subseteq \cup_{k\in\kappa(v_t)} S_k$. Since for any $k\in\kappa(v_t)$ all the element in 
$S_k$ (excluded $v_t$ itself) are distant at least $d$ from $v_t$, the $d$-sheltered property of the learner is 
ensured. Moreover, one can easily check that, for any $v$ and any $k\in\kappa(v)$, it holds that $\iota_k(v)\leq\iota(v)$. 
Thus, $H_t\subseteq \{v_s\,:\, s<t\}$, which means that the learner is only allowed to access past information, as 
required. We have hence proven that the strategy that we defined is admissible for a $d$-sheltered learner.

We now study the regret of this $d$-sheltered learner. For any $v$, let $r(v) = c_{\iota(v)}(a_{\iota(v)}, b_{\iota(v)}) - c_{\iota(v)}(a, b_{\iota(v)})$, and, for $v\in S_k$, define $r_k(v) = c_{\iota(v)}(a^{(k)}_{\iota_k(v)}, b_{\iota(v)}) - c_{\iota(v)}(a, b_{\iota(v)})$,
where we used that, for any $v\in S_k$, $c^{(k)}_{\iota_k(v)} = c_\iota(v)$ and $b^{(k)}_{\iota_k(v)} = b_{\iota(v)}$. With these definitions in mind we can rewrite
$$\Reg_{\Pi,n}(a)  = \sum_{v\in V(G)} r(v)\qquad\text{ and } \qquad \Reg_{\Pi_k,n_k}(a) = \sum_{v\in S_k} r_k(v)\,.$$
Now, notice that thanks to the convexity of the cost we have that
$$c_{\iota(v)}(a_{\iota(v)}, b_{\iota(v)}) \leq \sum_{k\in\kappa(v)}w_k c_{\iota(v)}(a^{(k)}_{\iota_k(v)}, b_{\iota(v)})\,,$$
by Jensen's inequality, as $\sum_{k\in\kappa(v)}w_k=1$. In particular,  $r(v)\leq \sum_{k\in\kappa(v)}w_kr_k(v) = \sum_{k}\one{v\in S_k} w_kr_k(v)$, and hence
$$\Reg_{\Pi,n}(a) = \sum_{v\in V(G)}r(v) \leq \sum_{v\in V(G)}\sum_{k}\one{v\in S_k} w_kr_k(v) = \sum_kw_k\sum_{v\in S_k} r_k(v) = \sum_kw_k \Reg_{\Pi_k,n_k}(a)\,.$$
The fact that $\Reg_{\Pi_k,n_k}(a)\leq F(n_k)$ yields
$$\Reg_{\Pi,n}(a) \leq \sum_k w_k F(n_k) = \left(\sum_k w_k\right)\sum_k\frac{w_k}{\sum_{k'}w_{k'}}F(n_k) \leq \left(\sum_k w_k\right)F\left( \frac{n}{\sum_{k}w_{k}}\right)\,,$$ where the last inequality follows from Jensen's inequality (since $F$ is concave) and  \eqref{eq:sumNk}. Finally, notice that what we have proven so far holds for any arbitrary $d$-stable fractional partition. In particular, we can select the partition such that $\sum_k w_k = \chif$, and hence obtain
$$\Reg_{\Pi,n}(a) \leq \chif F\big(n/\chif\big)\,,$$ which is the regret upper bound that we wanted to prove.\qed

\end{document}